\documentclass[pmlr]{jmlr}

\usepackage{mathpartir}

\usepackage{hyperref}
\usepackage{enumerate}
\usepackage{microtype}
\usepackage{thm-restate}
\usepackage{paralist}

\usepackage{tikz}

\newcommand{\naturals}{\mathbb{N}}
\newcommand{\alphabet}{\Sigma}
\newcommand{\alphabetf}{\widehat{\Sigma}}
\newcommand{\alphabett}{\Gamma}
\newcommand{\irr}{\mathcal{I}}

\newcommand{\lang}{L}
\newcommand{\production}{\rightarrow}
\newcommand{\deriv}{\Rightarrow}
\newcommand{\progressp}{\models}

\newcommand{\lefts}{\ell}
\newcommand{\rights}{r}

\newcommand{\angl}[1]{\left\langle#1\right\rangle}
\newcommand{\cc}[2]{{[#1]}_{#2}}

\newcommand{\pdat}[3]{\mathrel{\raisebox{-3pt}{$\xrightarrow{#1,\,#2/#3}$}}}

\newenvironment{proofsketch}{\par\noindent{\bfseries\upshape Proof\ sketch\ }}{\jmlrQED}

\newcommand{\lp}{\texttt{[}} % chktex 9
\newcommand{\rp}{\texttt{]}} % chktex 9

\title{Decision problems for Clark-congruential languages}
\author{%
    \Name{Makoto Kanazawa\nametag{\thanks{This work performed at the National Institute of Informatics in Tokyo, Japan.}\thanks{Supported by JSPS KAKENHI Grant Number 17K00026.}}} \addr{Hosei University, Tokyo, Japan} \Email{kanazawa@hosei.ac.jp}%
    \AND%
    \Name{Tobias Kapp\'{e}\nametag{\footnotemark[1]\thanks{Partially supported by the ERC Starting Grant ProFoundNet (grant code 679127).}}} \addr{University College London, London, United Kingdom} \Email{tkappe@cs.ucl.ac.uk}
}

\jmlrpages{}
\jmlrproceedings{}{}

\begin{document}

\maketitle

\begin{abstract}
A common question when studying a class of context-free grammars (CFGs) is whether equivalence is decidable within this class.
We answer this question positively for the class of \emph{Clark-congruential grammars}, which are of interest to grammatical inference.
We also consider the problem of checking whether a given CFG is Clark-congruential, and show that it is decidable given that the CFG is a deterministic CFG\@.
\end{abstract}

\section{Introduction}

Given two context-free grammars (CFGs), the \emph{equivalence problem} asks whether they represent the same language; this is well known to be undecidable in general~\cite{bar-hillel-perles-shamir-1961}.
In contrast, the equivalence problem is decidable within some families of CFGs, such as deterministic CFGs~\cite{senizergues-2001} and (pre-){}NTS grammars~\cite{senizergues-1985,autebert-boasson-1992}.
Thus, a reasonable question to ask when studying a subclass of CFGs is whether equivalence is decidable for members of this class.

One subclass of CFGs of interest to grammatical inference consists of the CFGs considered in~\cite{clark-2010b}, which we refer to as \emph{Clark-congruential (CC) grammars}.
There it is shown that, given an oracle called the ``teacher'', an algorithm can infer a language known to the teacher by posing questions about the language in a fixed format.
In particular, one type of question that the teacher can answer is an \emph{equivalence query}, where the algorithm supplies a CFG and asks whether it represents the language that the teacher has in mind.
A similar (if slightly less general) teacher can be used to infer regular languages~\cite{angluin-1987}.

In analogy to other classes of CFGs, one might ask whether the equivalence problem for CC grammars is decidable; in analogy to regular languages, one might ask whether it is in principle possible to implement a teacher that answers equivalence queries for a CC grammar.
Motivated by these questions, we investigate decision problems surrounding CC grammars.
Our main contribution is a proof that equivalence and congruence problems for these grammars are decidable, based on arguments of that ilk for pre-NTS grammars~\cite{autebert-boasson-1992}.
We also show that it is decidable whether a deterministic CFG is CC\@.

\medskip

The remainder of this paper is organised as follows.
In Section~\ref{section:preliminaries}, we recall some preliminary notions.
In Section~\ref{section:cc-languages}, we discuss the congruence, equivalence and recognition problems for CC grammars.
We list directions for further work in Section~\ref{section:further-work}.
To preserve the narrative, some proofs appear in the appendices.

\section{Preliminaries}%
\label{section:preliminaries}

A relation $R \subseteq S \times S$ is said to be \emph{Noetherian} if it does not admit an infinite chain, i.e., there exist no infinite sequence ${(s_n)}_{n \in \naturals}$ such that for all $n \in \naturals$ it holds that $s_n \neq s_{n+1}$ and $s_n \mathrel{R} s_{n+1}$.
$R$ is \emph{confluent on $S' \subseteq S$} if it is transitive and when for all $s, s', s'' \in S'$ such that $s \mathrel{R} s'$ and $s \mathrel{R} s''$, there exists a $t \in S'$ with $s' \mathrel{R} t$ and $s'' \mathrel{R} t$.

\paragraph{Words and languages}
We fix a finite set $\alphabet$, called the \emph{alphabet}, and write $\alphabet^*$ for the \emph{language} of words over $\alphabet$.
We write $\alphabett$ for another finite alphabet that contains $\alphabet$, and the symbol $\$$, which is not in $\alphabet$. % chktex 40
The \emph{empty word} is denoted by $\epsilon$.
We write $|w|$ for the \emph{length} of $w \in \alphabet^*$.
We also fix an (arbitrary) total order $\preceq$ on $\alphabet$, and extend $\preceq$ to an order on $\alphabet^*$ by defining $x \preceq y$ if and only if either $|x| < |y|$, or $|x| = |y|$ and $x$ precedes $y$ lexicographically.
A \emph{prefix} (resp.\ \emph{suffix}) of $w \in \alphabet^*$ is a $w' \in \alphabet^*$ such that there exists a $y \in \alphabet^*$ with $w'y = w$ (resp.\ $yw' = w$); $w$ \emph{overlaps} with $x$ if a non-empty suffix of $w$ is a prefix of $x$, or vice versa.

A function $h: \alphabet^* \to \alphabet^*$ is a \emph{morphism} when for $w, x \in \alphabet^*$ it holds that $h(wx) = h(w)h(x)$.
If we define a function $h: \alphabet \to \alphabet^*$, then $h$ uniquely extends to a morphism $h: \alphabet^* \to \alphabet^*$, by defining for $a_0, a_1, \dots, a_{n-1} \in \alphabet$ that $h(a_0a_1\cdots{}a_{n-1}) = h(a_0)h(a_1)\cdots{}h(a_{n-1})$.
If for all $a \in \alphabet$ we have that $h(a) \in \alphabet$, we say that $h$ is \emph{strictly alphabetic}.
When $L$ is a language, we write $h^{-1}(L)$ for the language given by $\{ w \in \alphabet^* : h(w) \in L \}$.

A \emph{semi-Thue system}~\cite{book-otto-1993} is a reflexive and transitive relation $\leadsto$ on $\alphabet^*$ such that if $w \leadsto w'$ and $x \leadsto x'$, then $wx \leadsto w'x'$.
A \emph{reduction} is a Noetherian semi-Thue system.
We say that $x \in \alphabet^*$ is \emph{irreducible} by a reduction $\leadsto$ if $x \leadsto x'$ implies that $x = x'$.

A \emph{congruence} is an equivalence $\sim$ on $\alphabet^*$ such that when $u \sim v$ and $w \sim x$, also $uw \sim vx$.
If $\sim$ is a congruence on $\alphabet^*$, we write $\cc{w}{\sim}$ for the \emph{congruence class} containing $w \in \alphabet^*$.
A congruence $\sim$ is \emph{finitely generated} if for some finite $S \subseteq \alphabet^* \times \alphabet^*$, $\sim$ is the smallest congruence containing $S$; the set $S$ is said to \emph{generate} $\sim$.
Any language $L$ induces a \emph{syntactic congruence}, denoted $\equiv_L$, which is the relation where $w \equiv_L x$ holds precisely when, for all $u, v \in \alphabet^*$, we have $uwv \in L$ if and only if $uxv \in L$.
The \emph{language of contexts} of $w \in \alphabet^*$ w.r.t.\ a language $L$, denoted $L[w]$, is $\{ u\sharp{}v : uwv \in L \}$ (for a distinguished symbol $\sharp$).
It should be clear that $w \equiv_L x$ if and only if $L[w] = L[x]$.

A language $L$ is \emph{congruential}~\cite{book-otto-1993} if there exists a finitely-generated congruence $\sim$ and a finite set $T \subseteq \Sigma^*$ such that $L = \bigcup_{t \in T} \cc{t}{\sim}$.
We say that $L$ is \emph{regular} if its syntactic congruence induces finitely many congruence classes~\cite{nerode-1958}.

Decidability of congruence and of equivalence are closely related for congruential languages, as witnessed by the following lemma from~\cite{senizergues-1985}.

\begin{lemma}%
\label{lemma:congruential-decidable}
Let $\sim_1$ and $\sim_2$ be congruences generated by finite sets $S_1, S_2 \subseteq \alphabet^* \times \alphabet^*$ respectively, and let $T_1, T_2 \subseteq \alphabet^*$ be finite.
Let $L_1$ and $L_2$ be given by
\begin{mathpar}
L_1 = \bigcup_{t \in T_1} \cc{t}{\sim_1}
\and
L_2 = \bigcup_{t \in T_2} \cc{t}{\sim_2}
\end{mathpar}
If we can decide $L_1$ and $L_2$, as well as $\equiv_{L_1}$ and $\equiv_{L_2}$, then we can decide whether $L_1 = L_2$.
\end{lemma}
\begin{proof}
Observe that $L_1 = L_2$ precisely when $T_1 \subseteq L_2$ and $T_2 \subseteq L_1$, as well as ${\sim_1} \subseteq {\equiv_{L_2}}$ and ${\sim_2} \subseteq {\equiv_{L_1}}$.
The first two inclusions are decidable, since $T_1$ and $T_2$ are finite, and $L_1$ and $L_2$ are decidable.
The latter two inclusions are also decidable, for they are equivalent to checking whether $S_1 \subseteq {\equiv_{L_2}}$ and $S_2 \subseteq {\equiv_{L_1}}$.
Thus, we can decide whether $L_1 = L_2$.
\end{proof}

\paragraph{Context-free grammars}
A \emph{(context-free) grammar} (\emph{CFG}) is a tuple $G = \angl{V, P, I}$ where $V$ is a finite set of symbols called \emph{nonterminals} with $I \subseteq V$ the \emph{initial nonterminals}, and $P \subseteq V \times {(\alphabet \cup V)}^*$ is a finite set of pairs called \emph{productions}.
We denote $\angl{A, \alpha} \in P$ by $A \production \alpha$.
We use $G$ to denote an arbitrary CFG $\angl{V, P, I}$, implicitly quantifying over all CFGs.

We write $\alphabetf$ for the set $\alphabet \cup V$ and define $\deriv_G$ as the smallest relation on $\alphabetf^*$ such that for all $\alpha, \gamma \in \alphabetf^*$ and $B \production \beta \in P$, we have $\alpha{}B\gamma \deriv_G \alpha\beta\gamma$.
For $\alpha \in \alphabetf^*$, the \emph{language of $\alpha$ in $G$}, denoted $\lang(G, \alpha)$ is $\{ w \in \alphabet^* : \alpha \deriv_G^* w \}$; the \emph{language of $G$}, denoted $\lang(G)$, is $\bigcup_{A \in I} \lang(G, A)$.
We say that $L \subseteq \alphabet^*$ is a \emph{context-free language} (\emph{CFL}) if $L = \lang(G)$ for some CFG $G$.

As an example of a CFG, let us fix $G_D = \angl{V_D, P_D, I_D}$ as a CFG over the alphabet $\{ \lp, \rp \}$, where $V_D = I_D = \{ S \}$, and $P_D$ contains the rules $S \production \epsilon$ and $S \production \lp{}S\rp$ and $S \production SS$.
The language of $G_D$ is the well-known \emph{Dyck language}, which consists of strings of well-nested parentheses, and which we shall use as a recurring example throughout this paper.

If $\lang(G, \alpha)$ is non-empty, we write $\vartheta_G(\alpha)$ for the $\preceq$-minimum of $\lang(G, \alpha)$.
Now, if $\lang(G, \alpha\beta)$ is non-empty, then $\vartheta_G(\alpha\beta) = \vartheta_G(\alpha)\vartheta_G(\beta)$.
We define $\leadsto_G$ as the smallest semi-Thue system such that whenever $A \production \alpha \in P$ and $\lang(G, \alpha) \neq \emptyset$, also $\vartheta_G(\alpha) \leadsto_G \vartheta_G(A)$.
As an example, for $G_D$ we see that $\vartheta_{G_D}(S) = \epsilon$, and hence $\leadsto_{G_D}$ is generated solely by the rule $\lp\rp \leadsto_{G_D} \epsilon$.

We observe that $\leadsto_G$ is a reduction (regardless of $G$), and that for all $A \in V$ and $w \in \lang(G, A)$ it holds that $w \leadsto_G \vartheta_G(A)$.
We write $\irr_G$ for the set of words irreducible by $\leadsto_G$.
Note that $\irr_G$ is regular: it is the complement of the regular language of words containing the left-hand side of a rule defining $\leadsto_G$, and regular languages are closed under complementation.
For instance, it is not hard to see that $\irr_{G_D} = \{ \rp^n\lp^m : n, m \geq 0 \}$.

We say that $G$ is \emph{weakly $\omega$-reduced} when for $A \in V \setminus I$ we have that $\lang(G, A)$ is infinite, and for all productions $A \production \alpha$ where $\lang(G, A)$ is finite, we have that $\alpha \in \alphabet^*$.

\begin{lemma}%
\label{lemma:cc-transformations}
\hspace{-2mm}\footnote{%
    Details appear in Appendix~\ref{appendix:transformations}.
}
Let $G = \angl{V, P, I}$ be a CFG, let $R$ be a regular language and let $h: \alphabet^* \to \alphabet^*$ be a strictly alphabetic morphism.
All of the following hold:
\begin{enumerate}[(i)]
    \setlength{\itemsep}{0em}
    \item
    We can construct a weakly $\omega$-reduced CFG $G_\omega = \angl{V_\omega, P_\omega, I_\omega}$ such that $\lang(G_\omega) = \lang(G)$ and $V_\omega \subseteq V$; moreover, when $A \in V_\omega$ it holds that $\lang(G, A) = \lang(G_\omega, A)$.

    \item
    We can construct a CFG $G^h = \angl{V^h, P^h, I^h}$ such that $\lang(G^h) = h^{-1}(\lang(G))$ and $V^h \subseteq V$; moreover, when $A \in V^h$ it holds that $h^{-1}(\lang(G, A)) = \lang(G^h, A)$.

    \item
    We can construct a CFG $G_R = \angl{V_R, P_R, I_R}$ such that $\lang(G_R) = \lang(G) \cap R$; moreover, when $A \in V_R$ there exist $A' \in V$ and $w \in \alphabet^*$ such that $\lang(G_R, A) = \lang(G, A') \cap \cc{w}{\equiv_R}$.
\end{enumerate}
\end{lemma}

\paragraph{Pushdown automata}
A \emph{pushdown automaton} (\emph{PDA}) is a tuple $M = \angl{Q, \rightarrow, q^0, F}$ where $Q$ is a finite set of \emph{states}, $q^0 \in Q$ is the \emph{initial state}, $F \subseteq Q$ are the \emph{accepting states} and $\rightarrow\ \subseteq Q \times (\alphabet \cup \{ \epsilon \}) \times \alphabett \times \alphabett^* \times Q$ is the (finite) \emph{transition relation}.
When $\angl{q, a, \sigma, \rho, q'} \in\ \rightarrow$, we write $q \pdat{a}{\sigma}{\rho} q'$.
The set of \emph{configurations} of $M$, denoted $\mathcal{C}_M$, is $Q \times \alphabet^* \times \alphabett^*$.
We define $\progressp_M$ as the smallest relation on $\mathcal{C}_M$ such that whenever $q \pdat{a}{\sigma}{\rho} q'$ and $w \in \alphabet^*$ as well as $\tau \in \alphabett^*$, it holds that $\angl{q, aw, \sigma\tau} \progressp_M \angl{q', w, \rho\tau}$.
The \emph{language} of $M$, denoted $\lang(M)$, is%
\footnote{%
    This definition is non-standard, in that upon acceptance the machine should be in an accepting state, \emph{and} the stack contains exactly $\$$. % chktex 40
    A (D){}PDA with this acceptance condition can easily be converted into an equivalent (D){}PDA with the standard acceptance condition, provided that its transitions preserve the end-of-stack marker; this is the case for all DPDAs in this paper.
    We omit details for the sake of brevity.
}
\[\left\{ w \in \alphabet^* : \angl{q^0, w, \$} \progressp_M^* \angl{q, \epsilon, \$},\ q \in F \right\}\]

$M$ is a \emph{deterministic PDA} if,
\begin{inparaenum}[(i)]
    \item for all $q \in Q$, $a \in \alphabet \cup \{ \epsilon \}$, and $\sigma \in \alphabett$, there is at most one $\rho \in \alphabett^*$ and at most one $q' \in Q$ such that $q \pdat{a}{\sigma}{\rho} q'$, and,
    \item for all $q' \in Q$ and $\rho \in \alphabett^*$ such that $q \pdat{\epsilon}{\sigma}{\rho} q'$, there are no $q'' \in Q$, $a \in \alphabet$ and $\rho' \in \alphabett^*$ such that $q \pdat{a}{\sigma}{\rho'} q''$.
\end{inparaenum}

If $M$ is a PDA and $L$ a language such that $\lang(M) = L$, we say that $M$ \emph{accepts} $L$.
It is well-known that a language is a CFL if and only if it is accepted by a PDA~\cite{chomsky-1962}.
A language accepted by a deterministic PDA is said to be a \emph{deterministic CFL} (\emph{DCFL}).
A CFG $G$ whose language is a DCFL is said to be a \emph{deterministic CFG} (DCFG).

As an example of a PDA, consider $M_D = \angl{\{ q \}, \rightarrow_D, q, \{ q \}}$, where $\rightarrow_D$ contains the rules $q \pdat{\lp}{\$}{\lp\$}_D q$, $q \pdat{\lp}{\lp}{\lp\lp}_D q$ and $q \pdat{\rp}{\lp}{\epsilon}_D q$.
This PDA happens to be deterministic, and it is not hard to see that it accepts the Dyck language, $\lang(G_D)$; this makes $G_D$ a DCFG\@.

\section{Clark-congruential languages}%
\label{section:cc-languages}

We now turn our attention to \emph{Clark-congruential languages}.
These are context-free languages that are defined by grammars where every nonterminal has a language that is contained in a congruence class of its grammar; more formally, we work with the following definition.

\begin{definition}
$G$ is \emph{Clark-congruential} (CC) if for all $A \in V$, there exists an $x_A \in \alphabet^*$ s.t. $\lang(G, A)$ is a subset of $\cc{x_A}{\equiv_{\lang(G)}}$.
A language $L$ is CC if $L = \lang(G)$ for a CC grammar $G$.
\end{definition}

As an example of a CC grammar, consider $G_D$.
There, we find that if $w \in \lang(G_D, S)$, then $w$ consists of a string of balanced parentheses; hence, if $uwv \in \lang(G_D, S)$, then $uv \in \lang(G_D, S)$, and vice versa.
Consequently, it holds that for $w \in \lang(G_D, S)$ we have $w \equiv_{\lang(G_D)} \epsilon$.

CC grammars can be seen as a generalization of pre-NTS grammars~\cite{autebert-boasson-1992}, which are themselves a generalization of NTS grammars~\cite{boasson-1980,senizergues-1985,boasson-senizergues-1985}.
While the class of CC \emph{grammars} strictly contains the class of pre-NTS grammars, and thus the class of pre-NTS languages is contained in the class of CC languages, it remains an open question whether this inclusion is strict on the level of languages; likewise, the class of NTS grammars is contained in the class of pre-NTS grammars, but the question of equal expressiveness remains open.

\subsection{Congruence and equivalence}

We now consider the question of deciding equivalence of CC grammars.
Our strategy here will be to verify the preconditions of Lemma~\ref{lemma:congruential-decidable} w.r.t.\ CC languages.
Thus, our first task is to show that all CC languages are congruential; this is indeed the case.

\begin{lemma}%
\label{lemma:clark-congruential-implies-congruential}
If $L$ is a CC language, then $L$ is congruential.
\end{lemma}
\begin{proof}
Let $G$ be a CC grammar such that $L = \lang(G)$ and choose $\sim$ as the smallest congruence containing $\leadsto_G$.
Obviously, $\sim$ is finitely generated.
We now claim that
\[
    \lang(G) = \bigcup_{A \in I,\, \lang(G, A) \neq \emptyset} \cc{\vartheta_G(A)}{\sim}
\]
For the inclusion from left to right, note that if $w \in \lang(G, A)$ for an $A \in I$, then $w \leadsto_G \vartheta_G(A)$, and hence $w \sim \vartheta_G(A)$; thus, $w \in \cc{\vartheta_G(A)}{\sim}$.
For the other inclusion, note that since $G$ is CC, ${\sim} \subseteq {\equiv_{\lang(G)}}$.
Hence, if $w \sim \vartheta_G(A)$ with $A \in I$, then $w \equiv_{\lang(G)} \vartheta_G(A)$, and thus $w \in \lang(G)$.
\end{proof}

We use $G$ to denote an arbitrary CC grammar, and set out to validate the second assumption of Lemma~\ref{lemma:congruential-decidable}, i.e., to show that if $G$ is CC, then $\equiv_{\lang(G)}$ is decidable.
To this end, we observe the following; details are in Appendix~\ref{appendix:transformations}.

\begin{lemma}
The grammar transformations from Lemma~\ref{lemma:cc-transformations} preserve Clark-congruentiality.
\end{lemma}

The algorithm that we describe to decide $\equiv_{\lang(G)}$ is essentially a generalization of the one found in~\cite{autebert-boasson-1992}.
Before we dive into formal details, it helps to sketch a high-level roadmap of the steps required to establish the desired result, in analogy with the steps in op.\ cit.
We proceed as follows:
\begin{enumerate}[(I)]
    \item\label{step:confluent}
    We argue that, when $G$ is CC, $\leadsto_G$ is almost confluent: it can be used to decide $w \in \lang(G)$ by reducing $w$ using any strategy, until we reach an irreducible word.

    \item\label{step:transform}
    We show that, for a given $w \in \alphabet^*$, we can use the transformations discussed earlier to construct a particular CC grammar $G_w$, which has a number of useful properties.

    \item\label{step:create-dpda}
    From $G_w$, we create a DPDA $M_w$ accepting a language very close to $\lang(G)[w]$; this DPDA exploits the almost-confluent nature of $\leadsto_{G_w}$ and the properties of $G_w$.

    \item\label{step:wrapping-up}
    We argue that $w \equiv_{\lang(G)} x$ if and only if $\lang(M_w) = \lang(M_x)$.
    Since the latter is decidable~\cite{senizergues-2001}, we can decide the former.
\end{enumerate}

\paragraph{Step~\eqref{step:confluent}: reduction is (almost) confluent}

If $G$ is pre-NTS, then $\leadsto_G$ is confluent on $\lang(G)$, but not necessarily on $\alphabet^*$~\cite{autebert-boasson-1992}.
For CC languages, this property is lost.
As an example, consider the CC grammar $G'$ with the rules $S \production aS$, $S \production a$, $T \production aaT$ and $T \production \epsilon$, and both $S$ and $T$ initial.
We find that $\vartheta_{G'}(S) = a$ and $\vartheta_{G'}(T) = \epsilon$, and hence $aa \leadsto_{G'} a$ as well as $aa \leadsto_{G'} \epsilon$, but both $a$ and $\epsilon$ are irreducible in $\leadsto_{G'}$.

On the positive side, $\leadsto_G$ is still useful in deciding membership of $\lang(G)$:
\begin{lemma}%
\label{lemma:derivations-vs-reductions}
There exists an $A \in I$ with $x \leadsto_G \vartheta_G(A)$ if and only if $x \in \lang(G)$.
\end{lemma}
\begin{proof}
For the direction from left to right, note that if $x \leadsto_G \vartheta_G(A)$, then $x \equiv_{\lang(G)} \vartheta_G(A)$, and therefore $x \in \cc{\vartheta_G(A)}{\equiv_{\lang(G)}}$.
Since $\vartheta_G(A) \in \lang(G)$, also $x \in \lang(G)$.
For the other direction, note that if $x \in \lang(G)$, then $x \in \lang(G, A)$ for some $A \in I$, and therefore $x \leadsto_G \vartheta_G(A)$.
\end{proof}

Using Lemma~\ref{lemma:derivations-vs-reductions}, we can simply apply reductions (using any strategy) to $w \in \alphabet^*$ from $\leadsto_G$, until we reach an irreducible word $w_r$.
This process terminates, since $\leadsto_G$ is Noetherian.
At that point, either $w_r = \vartheta_G(A)$ for some $A \in I$, in which case $w \in \lang(G)$, or $w_r \neq \vartheta_G(A)$ for all $A \in I$, in which case $w_r \not\in \lang(G)$ (since $w_r \in \irr_G$), and since $w \equiv_{\lang(G)} w_r$, also $w \not\in \lang(G)$.

As an example, consider the word $\lp\lp\rp\lp\rp\rp$, which can be reduced using $\leadsto_{G_D}$ as follows:
\[
    \lp\lp\rp\underline{\lp\rp}\rp\lp\rp
        \leadsto_{G_D} \lp\underline{\lp\rp}\rp\lp\rp
        \leadsto_{G_D} \lp\rp\underline{\lp\rp}
        \leadsto_{G_D} \underline{\lp\rp}
        \leadsto_{G_D} \epsilon = \vartheta_{G_D}(S)
\]
And hence $\lp\lp\rp\lp\rp\rp \in \lang(G_D)$.
On the other hand, the word $\lp\lp\rp$ can be reduced to $\lp$ only, and therefore Lemma~\ref{lemma:derivations-vs-reductions} allows us to conclude that $\lp\lp\rp \not\in \lang(G_D)$.

The (implicit) precondition that $G$ is CC is necessary to establish Lemma~\ref{lemma:derivations-vs-reductions}.
As an example, consider the grammar $G'$ with rules $S \production a$, $S \production b$ and $T \production ab$, with both $S$ and $T$ initial.
This grammar is not CC\@.
If we assume that $a \preceq b$, then $\leadsto_G$ is generated by the rule $b \leadsto_G a$.
We then find that $bb \leadsto_G ab$ and $\vartheta_G(T) = ab$, while $bb \not\in \lang(G)$.

\paragraph{Step~\eqref{step:transform}: construct $G_w$}

We now proceed to construct a CC grammar $G_w$ from $G$.
This is done by progressively applying the CC-preserving transformations described in Lemma~\ref{lemma:cc-transformations}.

First, we augment $\alphabet$ by adding for $a \in \alphabet$ the (unique) letter $a'$, i.e., every letter gains a ``primed'' version; this does not change $\lang(G)$, or the fact that $G$ is CC\@.
We write $\alphabet_0$ for the original alphabet, and $\alphabet_1$ for the set of newly added letters.
Moreover, let $h: \alphabet^* \to \alphabet^*$ be the morphism that removes the primes from $w \in \alphabet^*$, i.e., the morphism defined by setting $h(a) = a$ for $a \in \alphabet_0$ and $h(a') = a$ for $a' \in \alphabet_1$.
We write $w'$ for the ``primed copy'' of $w$, i.e., the unique element of $\alphabet_1^*$ such that $h(w') = w$.
We proceed to define $G_w$ in steps, as follows:
\begin{itemize}
    \item
    Let $G' = \angl{V', P', I'}$ be such that $\lang(G') = h^{-1}(\lang(G))$.

    \item
    Let $G_w' = \angl{V_w', P_w', I_w'}$ be such that $\lang(G_w') = \lang(G') \cap Rw'R$, where $R = \irr_G \cap \alphabet_0^*$.

    \item
    Let $G_w = \angl{V_w, P_w, I_w}$ be such that $\lang(G_w) = \lang(G_w')$, and $G_w$ is weakly $\omega$-reduced.
\end{itemize}

By Lemma~\ref{lemma:cc-transformations}, these grammars are CC\@.
Without trying to get ahead of ourselves, we note that $\lang(G_w)$ is already somewhat close to $L[w]$.
After all, we know that $\lang(G_w) = \{ uw'v : u, v \in \irr_G, uwv \in L \}$.
The difference between $L[w]$ and $\lang(G_w)$ comes down to having $\sharp$ or $w'$ separate the parts of the words, and whether those parts need to be in $\irr_G$.

Some analysis of $G_w'$ now gives us the following.

\begin{lemma}%
\label{lemma:unprimed-word-finite-language}
Let $A \in I_w'$.
If $\lang(G_w', A) \cap \alphabet_0^* \neq \emptyset$ and $w' \neq \epsilon$, then $\lang(G_w', A) = \{ \vartheta_G(A) \}$.
\end{lemma}
\begin{proof}
Suppose that $y \in \lang(G_w', A) \cap \alphabet_0^*$.
First note that we can (without loss of generality) find $u, v \in \alphabet^*$ such that $u \lang(G_w', A) v \subseteq \lang(G_w') \subseteq Rw'R$.
Consequently, there exist $p, q \in R$ such that $uyv = pw'q$.
Since $w' \neq \epsilon$, this means that $y$ is a substring of $p$ or $q$, and thus $y \in R$.
For the remainder, it suffices to show that $y = \vartheta_G(A)$, and $\lang(G_w', A) \setminus \alphabet_0^* = \emptyset$.

First, note that $y \in \lang(G', A)$, and so $h(y) = y \in \lang(G, A)$; thus, $y \leadsto_G \vartheta_G(A)$.
Since $y \in \irr_G$, we have $y = \vartheta_G(A)$.
Also, suppose towards a contradiction that $z \in \lang(G_w', A) \setminus \alphabet_0^*$.
Then $z$ contains at least one primed letter.
By choice of $u$ and $v$, we find that $uzv \in \lang(G_w')$.
Now $uzv$ contains strictly more primed letters than $uyv$; since all words in $\lang(G_w')$ contain exactly $|w'|$ primed letters, we have reached a contradiction.
We conclude that $\lang(G_w', A) \setminus \alphabet_0^* = \emptyset$.
\end{proof}

Since $G_w$ is the weakly $\omega$-reduced version of $G_w'$, we can show the following:

\begin{lemma}%
\label{lemma:reduction-pattern}
Let $A \production \alpha \in P_w$ with $\lang(G_w, \alpha) \neq \emptyset$. Then $\vartheta_{G_w}(A)$ and $\vartheta_{G_w}(\alpha)$ either contain or share an overlap with $w'$; more formally, one of the following holds:
\begin{enumerate}[(i)]
    \setlength{\itemsep}{0em}
    \item
    $\vartheta_{G_w}(A) = x_{A}w_\lefts'$ and $\vartheta_{G_w}(\alpha) = x_\alpha{}w_\lefts'$, for $x_A, x_\alpha \in \alphabet_0^*$ and $w_\lefts'$ a nonempty prefix of $w'$
    \item
    $\vartheta_{G_w}(A) = w_\rights'y_A$ and $\vartheta_{G_w}(\alpha) = w_\rights'y_\alpha$, for $y_A, y_\alpha \in \alphabet_0^*$ and $w_\rights'$ a nonempty suffix of $w'$
    \item
    $\vartheta_{G_w}(A) = x_{A}w'y_{A}$ and $\vartheta_{G_w}(\alpha) = x_\alpha{}w'y_\alpha$, for $x_A, y_A, x_\alpha, y_\alpha \in \alphabet_0^*$.
\end{enumerate}
\end{lemma}
\begin{proof}
If $\lang(G_w, A)$ is finite, then $A \in I_w$ (since $G_w$ is weakly $\omega$-reduced), and therefore $\vartheta_{G_w}(A), \vartheta_{G_w}(\alpha) \in \lang(G_w) \subseteq \irr_{G}w'\irr_{G}$; thus, $\vartheta_{G_w}(A)$ and $\vartheta_{G_w}(\alpha)$ satisfy the third condition.

Otherwise, suppose that $\lang(G_w, A)$ is infinite.
First, note that there exist $x, y \in \alphabet^*$ such that $x\lang(G_w, A)y \subseteq \lang(G_w)$.
Thus, there exist $u, v \in \irr_G \subseteq \alphabet_0^*$ such that $x\vartheta_{G_w}(A)y = uw'v$.
Suppose, towards a contradiction, that $\vartheta_{G_w}(A)$ neither contains nor overlaps with $w'$.
In that case, $\vartheta_{G_w}(A) \in \alphabet_0^*$, and $w' \neq \epsilon$; then, since $A \deriv_{G_w}^* \vartheta_{G_w}(A)$, also $A \deriv_{G_w'}^* \vartheta_{G_w}(A)$.
By Lemma~\ref{lemma:unprimed-word-finite-language}, we have that $\lang(G_w', A)$ is finite.
But since $\lang(G_w', A) = \lang(G_w, A)$ and the latter is infinite, we have a contradiction.
Therefore $\vartheta_{G_w}(A)$ must contain or overlap with $w'$.

Suppose $\vartheta_{G_w}(A) = x_{A}w_\lefts'$ for $x_A \in \alphabet_0^*$ and $w_\lefts'$ a nonempty prefix of $w'$; other cases are similar.
Write $w' = w_\lefts'w_\rights'$ and $y = w_\rights'v$.
By choice of $x$ and $y$, we have $x\vartheta_{G_w}(\alpha)w_\rights'v = x\vartheta_{G_w}(\alpha)y \in \lang(G_w) \subseteq \irr_{G}w'\irr_{G}$.
Therefore, $\vartheta_{G_w}(\alpha) = x_\alpha{}w_\lefts'$ for some $x_\alpha \in \alphabet_0^*$.
\end{proof}

This lemma tells us something about $\leadsto_{G_w}$: all of its generating rules overlap with $w'$, and moreover each rule preserves $w'$.
Thus, to decide whether $uw'v \in \lang(G_w)$, we can apply the rules of $\leadsto_{G_w}$ as described above; since every step involves (and preserves) part of $w'$, we also know that reductions must be clustered around the locus of $w'$.

\paragraph{Step~\eqref{step:create-dpda}: creating a DPDA}
The above analysis allows us to construct a DPDA that accepts $\{ u\sharp{}v : uw'v \in \lang(G_w) \}$, by going through the following phases:
\begin{enumerate}
    \item
    Read symbols and push them on the stack, until we encounter $\sharp$.

    \item
    From that point on, read from the stack or the input and apply reductions whenever possible, but with $\sharp$ standing in for the part of $w'$.

    \item
    When no reductions are possible (i.e., we have reached an element if $\irr_{G_w}$), check whether the buffer corresponds to a $\vartheta_{G_w}(A)$ for some $A \in I_w$.
\end{enumerate}
In the second step, the state of the DPDA holds a buffer to the left and the right of $\sharp$, large enough to detect any possible reductions.
Since $\leadsto_{G_w}$ is Noetherian, this phase must end after finitely many reductions; furthermore, since $\leadsto_{G_w}$ is length-decreasing, we can choose the size of the buffer appropriately.
Formally, this DPDA is defined as follows:

\begin{definition}
We build the PDA $M_w = \angl{Q, \rightarrow, q_0, F}$ as follows.
First, let $N$ be the maximum length of $\vartheta_{G_w}(\alpha)$ for $A \production \alpha$ in $G_w$.
Also, $Q$ and $F$ are the smallest sets satisfying
\begin{mathpar}
\inferrule{~}{%
    q_0 \in Q
}
\and
\inferrule{%
    u, v \in \alphabet_0^* \\
    |u|, |v| \leq N
}{%
    u\sharp{}v \in Q
}
\and
\inferrule{%
    A \in I_w \\
    \vartheta_{G_w}(A) = uw'v
}{%
    u\sharp{}v \in F
}
\end{mathpar}
Furthermore, $\rightarrow$ is the smallest transition relation satisfying
\begin{mathpar}
\inferrule{%
    a \neq \sharp
}{%
    q_0 \xrightarrow{b,\, a/ba} q_0
}
\and
\inferrule{~}{%
    q_0 \xrightarrow{\sharp,\, a/a} \sharp
}
\and
\inferrule{%
    u\sharp{}v \in Q \\
    |u| < N \\
    uw'v \in \irr_{G_w} \\
    a \neq \$
}{%
    u\sharp{}v \xrightarrow{\epsilon, a/\epsilon} au\sharp{}v
}
\and
\inferrule{%
    u\sharp{}v \in Q \\
    |v| < N \\
    uw'v \in \irr_{G_w} \\
    a = \$ \vee |u| = N
}{%
    u\sharp{}v \xrightarrow{b, a/a} u\sharp{}vb
}
\and
\inferrule{%
    u\sharp{}v \in Q \\
    uw'v \not\in \irr_{G_w} \\
    uw'v \leadsto_{G_w} xw'y\ \mbox{such that $xy$ is $\preceq$-minimal}
}{%
    u\sharp{}v \xrightarrow{\epsilon, a/a} x\sharp{}y
}
\end{mathpar}
\end{definition}

The first two rules take care of the first phase, where input is read onto the stack until we reach $\sharp$.
The third and fourth rule are responsible for reading symbols from the stack and from the input buffer respectively; the last rule applies reductions.
The set of accepting states makes sure that, upon acceptance, the buffer represents $\vartheta_{G_w}(A)$ for an $A \in I_w$.

We note that $M_w$ is deterministic: if $M_w$ is in state $q_0$, then the input is either equal to $\sharp$ (in which case the first rule applies) or not (in which case the second rule applies); otherwise, we are in some state $u\sharp{}v$, then either $uw'v \not\in \irr_{G_w}$ (and so the last rule applies), or the (mutually exclusive) third or fourth rule apply.

We can then show that $M_w$ indeed accepts $\{ u\sharp{}v : uw'v \in \lang(G_w) \}$.
We give a sketch of the proof below; details are in Appendix~\ref{appendix:constructed-dpda-language}.

\begin{restatable}{lemma}{dpdaacceptscontexts}%
\label{lemma:dpda-accepts-contexts}
$\lang(M_w) = \{ u\sharp{}v : uw'v \in \lang(G_w) \}$.
\end{restatable}
\begin{proofsketch}
For the inclusion from left to right, show that every change in configuration of $M_w$ corresponds to a step in the reduction of the input according to $\leadsto_{G_w}$, and that a configuration where $M_w$ accepts corresponds to this reduction reaching $\vartheta_{G_w}(A)$ for $A \in I_w$.

For the other inclusion, first note that if $u\sharp{}v$ is such that $uw'v \in \lang(G_w)$, we can let $M_w$ read up to and including $\sharp$, putting $u$ on the stack.
Subsequently, inspect the halting configuration reached by $M_w$ from that point on (which exists uniquely, for $\progressp_{M_w}$ is Noetherian), and show that it is a state where $M_w$ can accept --- i.e., that the remaining input and stack is empty, and that the buffer corresponds to an accepting state of $M_w$.
\end{proofsketch}

\paragraph{Step~\eqref{step:wrapping-up}: wrapping up}

Now we can show the following.
\begin{lemma}
$\lang(M_w) = \lang(M_x)$ if and only if $w \equiv_{\lang(G)} x$.
\end{lemma}
\begin{proof}
For the direction from left to right, suppose that $\lang(M_w) = \lang(M_x)$, and that $uwv \in \lang(G)$.
We can then find $u', v' \in \irr_G$ such that $u \leadsto_G u'$ and $v \leadsto_G v'$.
Now, since $G$ is CC and $u \equiv_{\lang(G)} u'$ and $v \equiv_{\lang(G)} v'$, we know that $u'wv' \in \lang(G)$.
Consequently, $u'\sharp{}v' \in \lang(M_w) = \lang(M_x)$, and therefore $u'xv' \in \lang(G)$, meaning that $uxv \in \lang(G)$.
By symmetry, $uxv \in \lang(G)$ also implies $uwv \in \lang(G)$; this allows us to conclude that $w \equiv_{\lang(G)} x$.

For the other direction, suppose that $y \in \lang(M_w)$.
Then $y = u\sharp{}v$ such that $u, v \in \irr_G$, and $uwv \in \lang(G)$.
Since $w \equiv_{\lang(G)} x$, it then follows that $uxv \in \lang(G)$, and thus $y = u\sharp{}v \in \lang(M_x)$.
This shows that $\lang(M_w) \subseteq \lang(M_x)$; the other inclusion follows symmetrically.
\end{proof}

The above characterises the syntactic congruence of $\lang(G)$ in terms of the equivalence of two DPDAs, constructible from $G$, $w$ and $x$.
Since equivalence of DPDAs is decidable~\cite{senizergues-2001}, it follows that we can decide $\equiv_{\lang(G)}$.
The main result then follows.
\begin{theorem}
It is decidable, given a CFG $G$ that is CC and $w, x \in \alphabet^*$, whether $w \equiv_{\lang(G)} x$.
It is furthermore decidable, given CFGs $G$ and $G'$ that are CC, whether $\lang(G) = \lang(G')$.
\end{theorem}

Like in~\cite{autebert-boasson-1992}, $M_w$ is \emph{one-turn}, i.e., it processes input first in a phase where the stack does not shrink (when it is still in $q^0$), and subsequently in a phase where the stack does not grow (in all other states).
Thus, an algorithm to test equivalence of finite-turn DPDAs~\cite{valiant-1974,beeri-1975} suffices.
Complexity-wise, this also helps: the equivalence problem for one-turn DPDAs is known to be in \textsc{co-np}~\cite{senizergues-2003}, while the problem for general DPDAs is known only to be primitive recursive~\cite{stirling-2002}.

\subsection{Recognition}

The \emph{recognition problem} for a class of CFGs $\mathcal{C}$ asks, given a CFG $G$, whether $G$ is in $\mathcal{C}$.
This problem is decidable for NTS grammars~\cite{senizergues-1985}, yet undecidable for a proper subclass of pre-NTS grammars~\cite{zhang-1992}.%
\footnote{%
    We note that the class of CFGs considered in~\cite{zhang-1992} was originally claimed to coincide with pre-NTS grammars~\cite{boasson-senizergues-1985}, but this is not strictly true: Zhang's class is a strict subclass of the pre-NTS grammars, although the languages that they can express are the same.
}

Given that our earlier decidability proofs were based on proofs of the same statement for pre-NTS grammars, one might ask whether we could extend the result from~\cite{zhang-1992} to CC grammars.
This turns out not to be the case.
The proof in op.\ cit.\ constructs, given a Turing machine $M$ and an input $w$, a CFG which is in the studied class if and only if $M$ does not halt on input $w$; this construction relies heavily on adding nonterminals with an empty language.
However, we can easily adapt the first construction from Lemma~\ref{lemma:cc-transformations} to show that we can remove all such nonterminals from a CFG $G$ to obtain an (equivalent) CFG $G'$; furthermore, $G$ is CC if and only if $G'$ is CC\@.
Thus, to decide whether a given CFG is CC, we can assume without loss of generality that no nonterminal has an empty language.
Hence, the undecidability proof from~\cite{zhang-1992} does not generalize to CC grammars.

We therefore turn our attention to finding a novel approach to the recognition problem for CC grammars, independent of (un){}decidability proofs of the recognition problem for its subclasses.
To this end, it is useful to introduce the following notion.
\begin{definition}
Let $\sim$ be a congruence.
$G$ is \emph{$\sim$-aligned} if, for every $A \in V$, there exists a $w_A \in \alphabet^*$ such that $\lang(G, A) \subseteq \cc{w_A}{\sim}$.
\end{definition}

Note that, by definition, $G$ is CC if and only if it is $\equiv_{\lang(G)}$-aligned.
As it turns out, $\sim$-alignment is decidable, provided that $\sim$ is decidable.
\begin{lemma}%
\label{lemma:decide-alignment}
Given a decidable congruence $\sim$, it is decidable whether a CFG $G$ is $\sim$-aligned.
\end{lemma}
\begin{proof}
Without loss of generality, assume that all nonterminals of $G$ have a non-empty language; if this is not the case, we can create a CFG $G'$ that does have this property, and which is $\sim$-aligned if and only if $G$ is.
Since $\vartheta_G: \alphabetf^* \to \alphabet^*$ is computable, it now suffices to prove that $G$ is $\sim$-aligned if and only if for all $A \production \alpha \in P$, it holds that $\vartheta_G(A) \sim \vartheta_G(\alpha)$.

For the direction from left to right, we know that if $A \production \alpha \in P$, then $\vartheta_G(A), \vartheta_G(\alpha) \in \lang(G, A) \subseteq \cc{w_A}{\sim}$ for some $w_A \in \alphabet^*$; hence, $\vartheta_G(A) \sim w_A \sim \vartheta_G(\alpha)$.
For the direction from right to left, a straightforward inductive argument shows that for all $\alpha, \beta \in \alphabetf^*$ such that $\alpha \deriv_G^* \beta$, we have that $\vartheta_G(\alpha) \sim \vartheta_G(\beta)$.
Hence, if $A \deriv_G^* w$, then we know that $\vartheta_G(A) \sim \vartheta_G(w) = w$, and thus it suffices to choose $w_A = \vartheta_G(A)$.
\end{proof}

As an application of the above, let $\sim$ be the smallest congruence on ${\{ \lp, \rp \}}^*$ such that $\lp\rp \sim \epsilon$.
Without too much effort, we can then show that we can \emph{uniquely} compute $m, n \in \naturals$ such that $w \sim \rp^m\lp^n$.
Therefore, we can conclude that $\sim$ is decidable: to decide whether $w \sim x$, check whether the $m$ and $n$ computed for $w$ are the same as the $m$ and $n$ computed for $x$.
Thus, by Lemma~\ref{lemma:decide-alignment}, we find that we can decide whether a given grammar $G$ over the alphabet ${\{ \lp, \rp \}}^*$ is $\sim$-aligned.
Indeed, $\sim$ turns out to be exactly $\equiv_{\lang(G_D)}$.

Lemma~\ref{lemma:decide-alignment} would also show that the recognition problem for CC grammars is decidable, provided that the congruence problem were decidable for arbitrary CFGs.
Unsurprisingly, this is not the case, as witnessed by the following lemma.
\begin{lemma}
It is undecidable, given a CFG $G$ and words $w, x \in \alphabet^*$, whether $w \equiv_{\lang(G)} x$.
\end{lemma}
\begin{proof}
We claim that $\lang(G) = \alphabet^*$ if and only if $\epsilon \in \lang(G)$, and for all $a \in \alphabet$ it holds that $a \equiv_{\lang(G)} \epsilon$.
First, suppose $\lang(G) = \alphabet^*$; then $\epsilon \in \lang(G)$ immediately.
Furthermore, for $a \in \alphabet$ and $u, v \in \alphabet^*$, we have that $uav, uv \in \lang(G)$, and thus $a \equiv_{\lang(G)} \epsilon$.
For the other direction, let $w \in \alphabet^*$.
An argument by induction on $|w|$ then shows that $w \equiv_{\lang(G)} \epsilon$, and hence $w \in \lang(G)$.

Since it is decidable whether $\epsilon \in \lang(G)$, the above equivalence tells us that we can decide $\lang(G) = \alphabet^*$ if we can decide the congruence problem for $G$.
Because the former is undecidable for CFGs in general~\cite{bar-hillel-perles-shamir-1961}, the claim follows.
\end{proof}

Fortunately, some classes of CFGs do have a decidable congruence problem.
This leads us to formulate our main result regarding the recognition problem, as follows.

\begin{theorem}
It is decidable, given a DCFG $G$, whether $G$ is CC\@.
\end{theorem}
\begin{proof}
Let us write $L = \lang(G)$.
By Lemma~\ref{lemma:decide-alignment}, it suffices to show that we can effectively obtain a decision procedure for $\equiv_L$.
We employ a technique similar to the method we used to decide $\equiv_L$ when $G$ is CC\@: we reduce the problem to checking equivalence of DCFLs.

Without loss of generality, let $\alphabet = \alphabet_0 \cup \{ \sharp \}$, with $\sharp \not \in \alphabet_0$, such that $L \subseteq \alphabet_0^*$.
For $w \in \alphabet^*$, we define the morphism $g_w: \alphabet^* \to \alphabet^*$ by setting $g_w(\sharp) = w$ and $g(a) = a$ for $a \in \alphabet_0$.

We now claim that $L[w] = g_w^{-1}(L) \cap \alphabet_0^*\sharp\alphabet_0^*$.
To see this, suppose that $u\sharp{}v \in L[w]$; then, since $g_w(u\sharp{}v) = uwv \in L$ and $u\sharp{}v \in \alphabet_0^*\sharp\alphabet_0^*$, we find that $u\sharp{}v \in g_w^{-1}(L)$.
For the other inclusion, suppose that $x \in g_w^{-1}(L) \cap \alphabet_0^*\sharp\alphabet_0^*$.
Since $x \in \alphabet_0^*\sharp\alphabet_0^*$, we can write $x = u\sharp{}v$ for $u,v \in \alphabet_0^*$.
Since $uwv = g_w(u\sharp{}v) = g(x) \in L$, we find that $u\sharp{}v \in L[w]$.

Since $L$ is a DCFL, we have a DPDA $M$ such that $L = \lang(M)$.
Furthermore, because DCFLs are closed under inverse morphism and intersection with regular languages~\cite{ginsburg-greibach-1966}, we can create for $w \in \alphabet^*$ a DPDA $M_w$ such that $\lang(M_w) = L[w]$.
Since it is decidable whether $\lang(M_w) = \lang(M_x)$~\cite{senizergues-2001}, we can decide whether $L[w] = L[x]$, and hence whether $w \equiv_L x$.
\end{proof}

\section{Further work}%
\label{section:further-work}

With regard to implementing a teacher for a given CC language, one detail remains to be settled.
The algorithm to learn CC languages from~\cite{clark-2010b} assumes the presence of an \emph{extended MAT}, in which the representation of the language in the equivalence query need not guarantee that the hypothesis language is in the class of languages being learned.
More concretely, this means that the algorithm might query the teacher with grammars that are not CC, and thus the decision procedure outlined in this paper need not apply.
Consequently, we wonder whether the learning algorithm can be adapted to work with a (proper) MAT, or alternatively, whether the decision procedure of this paper can be extended to accommodate the class of grammars that can be produced by the learning algorithm.

One possible direction for generalization of the decision procedure is the setting of \emph{multiple context-free grammars} (\emph{MCFGs})~\cite{seki-matsumura-fujii-kasami-1991}.
A notion corresponding to Clark-congruentiality for MCFGs is already known, and the class of languages generated by such MCFGs is also known to be learnable~\cite{yoshinaka-clark-2010}.
We conjecture that the decidability results can be lifted to Clark-congruential MCFGs, and that such a lifting would employ \emph{$n$-turn} DPDAs instead of one-turn DPDAs.

Equivalence and congruence are decidable for both DCFLs and CC languages.
To see if the case for CC languages follows from the case for DCFLs, one would have to investigate whether all CC grammars define a DCFL\@.
For what it's worth, the fact that we can decide whether a DCFG is CC appears to at least not contradict this possibility, and we have been unsuccessful in finding a counterexample thus far.

The question about the connection between CC languages and DCFLs can be seen as analogous to the (open) question of whether all pre-NTS grammars define a DCFL~\cite{autebert-boasson-1992}.
Since all NTS grammars are pre-NTS, and all pre-NTS grammars are in turn CC, it follows that every NTS language is a pre-NTS language, and in turn every pre-NTS language is a CC language; whether this inclusion is strict remains an open question.
It has been conjectured that these families of languages coincide~\cite{clark-2010b}.

\paragraph{Acknowledgements}
We would like to thank the anonymous referees of LearnAut and ICGI for their comments, which helped improve this paper.

\bibliography{bibliography.bib}

\newpage
\appendix

\section{The language of \texorpdfstring{$M_w$}{the constructed DPDA}}%
\label{appendix:constructed-dpda-language}

To analyze the behavior of $M_w$, we first note that if it is in a configuration with a state of the form $u\sharp{}v$, then all reachable configurations are related to that configuration by $\leadsto_{G_w}$.
In effect, this shows that $M_w$ proceeds according to $\leadsto_{G_w}$.

\begin{lemma}%
\label{lemma:pda-simulates-reduction}
If $u_0, u_1, v_0, v_1, x_0, x_1, y_0, y_1 \in \alphabet^*$ s.t. $\angl{u_0\sharp{}v_0, y_0, x_0^R\$} \progressp_{M_w} \angl{u_1\sharp{}v_1, y_1, x_1^R\$}$ then it follows that $x_0u_0w'v_0z_0 \leadsto_{G_w} x_1u_1w'v_1y_1$.
\end{lemma}
\begin{proof}
There are three cases to consider.
First, if $uw'v$ is reducible, then $u_0w'v_0 \leadsto_{G_w} u_1w'v_1$, as well as $y_0 = y_1$ and $x_0 = x_1$; the claim then follows.
Second, if $uw'v$ is irreducible and $x_0$ is non-empty, with $|u_0| < N$, then $x_0 = x_1a$ and $u_1 = au_0$, as well as $y_0 = y_1$ and $v_0 = v_1$; we derive that
$x_0u_0w'v_0y_0 = x_1au_0w'v_0y_0 = x_1u_1w'v_0y_0 = x_1u_1w'v_1y_1$.
Lastly, if $uw'v$ is irreducible and either $x_0$ is empty or $|u_0| = N$, then $ay_1 = y_0$ and $v_1 = v_0a$, as well as $x_0 = x_1$ and $u_0 = u_1$; thus, $x_0u_0w'v_0y_0 = x_0u_0w'v_0ay_1 = x_0u_0w'v_1y_1 = x_1u_1w'v_1y_1$.
\end{proof}

With this in hand, we can show that $M_w$ accepts the desired language.

\dpdaacceptscontexts*
\begin{proof}
For the inclusion from left to right, suppose that $x \in \lang(M_w)$.
We then know that $\angl{q_0, x, \$} \progressp_{M_w}^* \angl{u_1\sharp{}v_1, \epsilon, \$}$ such that there exists an $A \in I$ with $\vartheta_{G_w}(A) = u_1w'v_1$.
Thus, $x = u_0\sharp{}v_0$ such that $\angl{q_0, u_0\sharp{}v_0, \$} \progressp_{M_w}^* \angl{\sharp, v_0, u_0^R\$} \progressp_{M_w}^* \angl{u_1\sharp{}v_1, \epsilon, \$}$.
By Lemma~\ref{lemma:pda-simulates-reduction}, we have $u_0w'v_0 \leadsto_{G_w} u_1w'v_1 = \vartheta_{G_w}(A)$.
By Lemma~\ref{lemma:derivations-vs-reductions}, also $u_0w'v_0 \in \lang(G_w, A) \subseteq \lang(G_w)$.

For the inclusion from right to left, suppose that $u, v \in \alphabet^*$ are such that $uw'v \in \lang(G_w)$; our aim is to show that $u\sharp{}v \in \lang(M_w)$.
By construction of $M_w$, this DPDA first processes the input up to $\sharp$ to reach $C^\sharp = \angl{\sharp, \epsilon, v, u^R\$}$.

Let $C = \angl{u_1\sharp{}v_1, y, z^R\$}$ be the unique halting configuration of $M_w$ starting from $C^\sharp$; this configuration exists uniquely, because every transition of $M_w$ either advances the input, or performs a reduction using $\leadsto_{G_w}$.
We then have that $u_1w'v_1 \in \irr_{G_w}$, otherwise $C$ would not be halting.
Now, we observe that
\begin{inparaenum}[(i)]
    \item
    either $z$ is empty, or $|u_1| = N$ --- for otherwise $M_w$ could pop letters off the stack into the left buffer, meaning that $C$ would not be halting, and
    \item
    either $y$ is empty, or $|v_1| = N$ --- for otherwise $M_w$ could consume letters from the input into the right buffer, and so $C$ would again not be halting.
\end{inparaenum}

A reducible substring of $zu_1w'v_1y$ must start at least $N$ positions before the start of $w'$, and end at most $N$ positions from the end of $w'$ (by Lemma~\ref{lemma:reduction-pattern}).
Consequently, $\leadsto_{G_w}$ cannot reduce
\begin{inparaenum}[(a)]
    \item
    a substring overlapping $z$ --- otherwise $|u_1| < N$ and $z \neq \epsilon$, nor
    \item
    a substring overlapping $y$ --- otherwise $|v_1| < N$ and $y \neq \epsilon$.
\end{inparaenum}
Thus, if $zu_1w'v_1y$ were reducible, then the reducible substring must occur in $u_1w'v_1$ --- but this is a contradiction, since $u_1w'v_1 \in \irr_{G_w}$; hence, $zu_1w'v_1y$ is irreducible.

By Lemma~\ref{lemma:pda-simulates-reduction}, we know that $uw'v \leadsto_{G_w} zu_1w'v_1y$; also, by (the proof of) Lemma~\ref{lemma:clark-congruential-implies-congruential}, we have that $uw'v \equiv_{\lang(G_w)} zu_1w'v_1y$, and hence $zu_1w'v_1y \in \lang(G_w)$.
By Lemma~\ref{lemma:derivations-vs-reductions}, it follows that there exists an $A \in I$ such that $zu_1w'v_1y \leadsto_{G_w} \vartheta_{G_w}(A)$.
Consequently, $\vartheta_{G_w}(A) = zu_1w'v_1y$, and so either $|u_1| < N$, and thus $z = \epsilon$, or $|u_1| = N$, in which case $z = \epsilon$ again, as $|\vartheta_{G_w}(A)| \leq N$.
By a similar argument, we find that $y = \epsilon$.
But then $\vartheta_{G_w}(A) = u_1w'v_1$, and thus $u_1\sharp{}v_1 \in F$.
We can conclude that $u\sharp{}v \in \lang(M_w)$.
\end{proof}

\section{Transformations of CFGs}\label{appendix:transformations}

\begin{lemma}%
\label{lemma:weak-omega-reduction-construction}
We can construct a CFG $G_\omega$ using nonterminals of $G$, such that
\begin{inparaenum}[(i)]
    \item\label{lemma:weak-omega-reduction-construction:languages}
    if $A$ is a nonterminal of $G_\omega$, then $\lang(G_\omega, A) = \lang(G, A)$, and
    \item\label{lemma:weak-omega-reduction-construction:language}
    $\lang(G_\omega) = \lang(G)$, and
    \item\label{lemma:weak-omega-reduction-construction:soundness}
    $G$ is weakly $\omega$-reduced, and
    \item\label{lemma:weak-omega-reduction-construction:preservation}
    if $G$ is CC, then so is $G_\omega$.
\end{inparaenum}
\end{lemma}
\begin{proof}
We choose $G_\omega = \angl{V_\omega, P_\omega, I}$, where $V_\omega$ and $P_\omega$ are the smallest sets satisfying
\begin{mathpar}
\inferrule{%
    A \in V \\
    \lang(G, A)\ \mathrm{infinite}
}{%
    A \in V_\omega
}
\and
\inferrule{%
    A \in I
}{%
    A \in V_\omega
}
\and
\inferrule{%
    A \production \alpha \in P \\
    A \in V_\omega \\
    \alpha' \in \nu(\alpha)
}{%
    A \production \alpha' \in P_\omega
}
\end{mathpar}
and in which $\nu: \alphabetf^* \to 2^{{(V_\omega \cup \alphabet)}^*}$ is the substitution induced by setting for $\alpha \in \alphabetf$:
\[
\nu(\alpha) =
\begin{cases}
\lang(G, \alpha)     & \lang(G, \alpha)\ \mathrm{finite} \\
\{ \alpha \}         & \lang(G, \alpha)\ \mathrm{infinite} \\
\end{cases}
\]
Note that $G$ is a proper CFG\@; in particular, $P_\omega$ is finite, since if $\alpha \in \alphabetf^*$, then $\nu(\alpha)$ is finite.

We first argue claim~\eqref{lemma:weak-omega-reduction-construction:languages}, which immediately implies~\eqref{lemma:weak-omega-reduction-construction:language}.
Let $A \in V_\omega$.
\begin{itemize}
    \item[($\subseteq$)]
    Let $A \deriv_{G_\omega}^n w$ for some $n \in \naturals$.
    The proof proceeds by induction on $n$.
    In the base, where $n = 1$, we have $A \production w \in P_\omega$.
    By construction of $P_\omega$, there exists an $A \production \alpha \in P$ such that $w \in \nu(\alpha)$, i.e., such that $\alpha \deriv_G^* w$; thus, $A \deriv_G^* w$.

    For the inductive step, let $n > 0$ and assume that the claim holds for $n' < n$.
    We then find $\alpha' \in {(V_\omega \cup \alphabet)}^*$ such that $A \deriv_{G_\omega} \alpha' \deriv_{G_\omega}^{n-1} w$, with $A \production \alpha' \in P_\omega$.
    By construction of $P_\omega$, we know that $\alpha' \in \nu(\alpha)$ for some $A \production \alpha \in P$.
    Let us write $\alpha = a_0A_0a_1A_1\cdots{}A_{k-1}a_k$; we then know that $\alpha' = a_0\alpha_0a_1\alpha_1\cdots{}\alpha_{k-1}a_k$ such that for $0 \leq i < k$ it holds that $A_i \deriv_G^* \alpha_i$ for $n_i \leq n-1$.
    We can then write $w = a_0w_0a_1w_1\cdots{}w_{k-1}a_k$ such that for $0 \leq i < k$ we know that $\alpha_i \deriv_{G_\omega}^{n_i} w_i$ with $n_i \leq n-1$.
    By induction, we obtain that $\alpha_i \deriv_G^* w_i$ as well.
    In total,
    \[A \deriv_G a_0A_0a_1A_1\cdots{}A_{k-1}a_k \deriv_G^* a_0\alpha_0a_1\alpha_1\cdots{}\alpha_{k-1}a_k \deriv_G^* a_0w_0a_1w_1\cdots{}w_{k-1}a_k = w.\]

    \item[($\supseteq$)]
    Let $A \deriv_G^n w$ for some $n \in \naturals$.
    We proceed by induction on $n$.
    In the base, where $n = 1$, we have $A \production w \in P$.
    It follows that $w \in \nu(w)$, and so $A \production w \in P_\omega$, thus $A \deriv_{G_\omega}^* w$.

    For the inductive step, let $n > 1$ and assume that the claim holds for $n' < n$.
    We then find $\alpha \in {(V_\omega \cup \alphabet)}^*$ such that $A \deriv_G \alpha \deriv_G^{n-1} w$, with $A \production \alpha \in P$.
    We can write $\alpha = a_0A_0a_1A_1\cdots{}A_{k-1}a_k$ and $w = a_0w_0a_1w_1\cdots{}w_{k-1}a_k$ such that for $0 \leq i < k$ it holds that $A_i \deriv_G^* w_i$.
    For $0 \leq i < k$, we now choose $\alpha_i = A_i$ if $\lang(G, A_i)$ is infinite, and $\alpha_i = w_i$ otherwise.
    It follows that $\alpha' = a_0\alpha_0a_1\alpha_1\cdots\alpha_{k-1}a_k \in \nu(\alpha)$, which means that $A \production \alpha' \in P_\omega$.
    Furthermore, note that for $0 \leq i < k$, it holds that $\alpha_i \deriv_{G_\omega}^* w_i$ (where we apply the induction hypothesis for the case where $\lang(G, A_i)$ is infinite).
    Consequently, $A \deriv_{G_\omega} \alpha' \deriv_{G_\omega}^* w$.
\end{itemize}

As for~\eqref{lemma:weak-omega-reduction-construction:soundness}, note that if $A \in V_\omega \setminus I$, then $\lang(G_\omega, A) = \lang(G, A)$ is infinite by construction.
Also, if $A \production \alpha \in P_\omega$ and $\lang(G_\omega, A)$ is finite, then so is $\lang(G, A)$; since $\alpha \in \nu(\alpha) = \lang(G, \alpha)$, also $\alpha \in \alphabet^*$.
We can thus conclude that $G_\omega$ is $\omega$-reduced.
Lastly, for~\eqref{lemma:weak-omega-reduction-construction:preservation}, it suffices to observe that for $A \in V_\omega$ we have $\lang(G_\omega, A) = \lang(G, A) \subseteq \cc{\vartheta_G(A)}{\equiv_{\lang(G)}} = \cc{\vartheta_{G_\omega}(A)}{\equiv_{\lang(G_\omega)}}$.
\end{proof}

\begin{lemma}%
\label{lemma:inverse-strictly-alphabetic-morphism-construction}
Let $h: \alphabet^* \to \alphabet^*$ be a strictly alphabetic morphism.%
\footnote{With a little effort, this proof can be adapted to work for general alphabetic morphisms; the trick is to add a symbol that can generate all words over letters that are mapped to $\epsilon$ by $h$, and to intersperse this symbol in the right-hand sides of the productions of $G^h$.}
We can construct a CFG $G^h$ using nonterminals of $G$, such that
\begin{inparaenum}[(i)]
    \item\label{lemma:inverse-strictly-alphabetic-morphism-construction:languages}
    if $A$ is a nonterminal of $G^h$, then $\lang(G^h, A) = h^{-1}(\lang(G, A))$, and
    \item\label{lemma:inverse-strictly-alphabetic-morphism-construction:language}
    $\lang(G^h) = h^{-1}(\lang(G))$, and
    \item\label{lemma:inverse-strictly-alphabetic-morphism-construction:preservation}
    if $G$ is CC, then so is $G^h$.
\end{inparaenum}
\end{lemma}
\begin{proof}
First, let us extend $h$ to $\hat{h}: \alphabetf^* \to \alphabetf^*$ in the following way:
\[
\hat{h}(\alpha) =
\begin{cases}
h(\alpha)  & \alpha \in \alphabet \\
\alpha     & \alpha \in V
\end{cases}
\]
We construct the grammar $G^h = \angl{V, P^h, I}$, where $P^h = \{ A \production \alpha : A \production \hat{h} \in P \}$.
Note that $G$ is a proper CFG\@; in particular, $P^h$ is finite, since if $\alpha \in \alphabetf^*$, then there are only finitely many $\alpha' \in \alphabetf^*$ such that $h(\alpha') = \alpha$, since $h$ is strictly alphabetic.

We now pursue two sub-claims, as follows.
\begin{itemize}
    \item
    Let $A \in V$, and suppose that $A \deriv_G^n \hat{h}(\beta)$ for some $n \in \naturals$ and $\beta \in \alphabetf^*$; we claim that $A \deriv_{G^h}^* \beta$.
    The proof proceeds by induction on $n$.
    In the base, where $n = 0$, we know that $\hat{h}(\beta) = A$, and therefore $A = \beta$; it immediately follows that $A \deriv_{G^h}^* \beta$.
    For the inductive step, let $n > 0$ and assume the claim holds for $n' < n$.
    We find an $A \production \alpha \in P$ such that $\alpha \deriv_G^{n-1} \hat{h}(\beta)$.
    Let us write $\alpha = w_0A_0w_1A_1\cdots{}A_{m-1}w_m$ such that $w_0, w_1, \dots, w_m \in \alphabet^*$ and $A_0, A_1, \dots, A_{m-1} \in V$.
    Since $\hat{h}$ is strictly alphabetic, we can write $\beta$ as $x_0\beta_0x_1\beta_1\cdots{}\beta_{m-1}x_m$ such that for $0 \leq i < m$ we have that $A_i \deriv_G^{n_i} \hat{h}(\beta_i)$ for some $n_i \leq n-1$, and for $0 \leq i \leq m$ we have that $\hat{h}(x_i) = w_i$.
    By induction, we have for $0 \leq i < m$ that $A_i \deriv_{G^h}^* \beta_i$.
    Now, choose $\alpha' = x_0A_0x_1A_1\cdots{}A_{m-1}x_m$ and note that $\hat{h}(\alpha') = \alpha$; consequently, $A \deriv_{G^h} \alpha' \deriv_{G^h}^* x_0\beta_0x_1\beta_1\cdots\beta_{m-1}x_m = \beta$.

    \item
    Conversely, suppose that $A \deriv_{G^h}^n \beta$ for some $n \in \naturals$; we claim that $A \deriv_G^* \hat{h}(\beta)$.
    The proof proceeds by induction on $n$.
    In the base, where $n = 0$, we have that $A = \hat{h}(\beta)$, and therefore $\beta = A$; it immediately follows that $A \deriv_{G^h}^* \beta$.
    For the inductive step, let $n > 0$ and assume the claim holds for $n' < n$.
    We find an $A \production \alpha \in P^h$ such that $A \production \hat{h}(\alpha) \in P$ and $\alpha \deriv_{G^h}^{n-1} \beta$.
    Let us write $\alpha = w_0A_0w_1A_1\cdots{}A_{m-1}w_m$ such that $w_0, w_1, \dots, w_m \in \alphabet^*$ and $A_0, A_1, \dots, A_{m-1} \in V$.
    We can also write $\beta = w_0\beta_0w_1\beta_1\cdots\beta_{m-1}w_m$ such that for $0 \leq i < m$ it holds that $A_i \deriv_{G^h}^{n_i} \beta_i$ with $n_i \leq n-1$.
    By induction, we have for $0 \leq i < n$ that $A_i \deriv_{G^h}^* \hat{h}(\beta_i)$; thus, it follows that
    \begin{align*}
    A   &\deriv_G \hat{h}(\alpha) = \hat{h}(w_0)A_0\hat{h}(w_1)A_1\cdots{}A_{m-1}\hat{h}(w_m) \\
        &\deriv_G \hat{h}(w_0)\hat{h}(\beta_0)\hat{h}(w_1)\hat{h}(\beta_1)\cdots\hat{h}(\beta_{m-1})\hat{h}(w_m) = \hat{h}(\beta)
    \end{align*}
\end{itemize}
From the above, claims~\eqref{lemma:inverse-strictly-alphabetic-morphism-construction:languages} and~\eqref{lemma:inverse-strictly-alphabetic-morphism-construction:language} follow quite easily.

As for~\eqref{lemma:inverse-strictly-alphabetic-morphism-construction:preservation}, it suffices to show that if $A \in V$ and $w, x \in \lang(G^h, A)$, then $w \equiv_{\lang(G^h)} x$.
To this end, suppose that $u, v \in \alphabet^*$ such that $uwv \in \lang(G^h)$.
In that case, $h(uwv) = h(u)h(w)h(x) \in \lang(G)$.
Since $h(w), h(x) \in \lang(G, A)$ by~\eqref{lemma:inverse-strictly-alphabetic-morphism-construction:languages}, $h(w) \equiv_{\lang(G)} h(x)$ by the premise that $G$ is CC\@.
Consequently, $h(uxv) = h(u)h(x)h(v) \in \lang(G)$, and thus $uxv \in \lang(G^h)$.
Symmetrically, $uxv \in \lang(G^h)$ implies that $uwv \in \lang(G^h)$; we can thus conclude that $w \equiv_{\lang(G^h)} x$.
\end{proof}

\begin{lemma}%
\label{lemma:regular-intersection-construction}
Let $R$ be a regular language.
We can construct a CFG $G_R$ such that
\begin{inparaenum}[(i)]
    \item\label{lemma:regular-intersection-construction:languages}
    for every nonterminal $A$ of $G_R$, there exist $A' \in V$ and $x \in \alphabet^*$ such that $\lang(G_R, A) = \lang(G, A) \cap \cc{x}{\equiv_R}$, and
    \item\label{lemma:regular-intersection-construction:language}
    $\lang(G_R) = \lang(G) \cap R$, and
    \item\label{lemma:regular-intersection-construction:preservation}
    if $G$ is CC, then so is $G_R$.
\end{inparaenum}
\end{lemma}
\begin{proof}
For every congruence class $\cc{x}{\equiv_R}$ of $R$, pick a representative $x$; let $C$ be the set of these representatives.
Note that $C$ is finite, by the premise that $R$ is regular.
We construct the grammar $G_R = \angl{V_R, P_R, I_R}$, where $V_R$, $P_R$ and $I_R$ are the smallest sets satisfying
\begin{mathpar}
\inferrule{%
    x \in \lang(G, A) \cap C \\
    A \in V
}{%
    A^x \in V_R
}
\and
\inferrule{%
    x \in R \cap C \\
    A \in I
}{%
    A^x \in I_R
}
\and
\inferrule{%
    a_0x_0a_1x_1\cdots{}x_{n-1}a_n \equiv_R x \\
    A \production a_0A_0a_1A_1\cdots{}A_{n-1}a_n \in P
}{%
    A^x \production a_0A_0^{x_0}a_1A_1^{x_1}\cdots{}A_{n-1}^{x_{n-1}}a_n \in P_R
}
\end{mathpar}
Note that $G$ is a proper CFG\@; in particular, $P_R$ is finite, since $C$ and $P$ are finite.

We now argue claim~\eqref{lemma:regular-intersection-construction:languages}; more specifically, we claim that for $A^x \in V_R$ we have that $\lang(G_R, A^x) = \lang(G, A) \cap \cc{x}{\equiv_R}$.
From this, claim~\eqref{lemma:regular-intersection-construction:language} follows immediately.
\begin{itemize}
    \item[($\subseteq$)]
    Suppose that $A^x \deriv_{G_R}^n w$ for some $n \in \naturals$.
    We prove that $w \in \lang(G, A) \cap \cc{x}{\equiv_R}$ by induction on $n$.
    In the base, where $n = 1$, we have that $A^x \production w \in P_R$, thus $w \equiv_R x$ and $A \production w \in P$ by construction of $P_R$.
    Consequently, $w \in \cc{x}{\equiv_R}$ and $w \in \lang(G, A)$.

    For the inductive step, let $n > 1$, and assume the claim holds for $n' < n$.
    We then find that $A^x \deriv_{G_R} \alpha \deriv_{G_R}^{n-1} w$ for $A^x \production \alpha \in P_R$.
    By construction of $P_R$, we know that $\alpha = a_0A_0^{x_0}a_1A_1^{x_1}\cdots{}A_{k-1}^{x_{k-1}}a_k$ such that $a_0x_0a_1x_1\cdots{}x_{k-1}a_k \equiv_R x$ and $A \production a_0A_0a_1A_1\cdots{}A_{k-1}a_k \in P$.
    From this, we can derive that $w = a_0w_1a_1w_1\cdots{}w_{k-1}a_k$ such that for $0 \leq i < n$ it holds that $A_i^{x_i} \deriv_{G_R}^{n_i} w_i$ for some $n_i \leq n-1$.
    By induction, we know that for $0 \leq i < n$ it holds that $w_i \in \lang(G, A_i) \cap \cc{x_i}{\equiv_R}$.
    From this, it follows that $A \deriv_G^* w$ and $w = a_0w_0a_1w_1\cdots{}w_{k-1}a_k \equiv_R a_0x_0a_1x_1\cdots{}x_{k-1}a_k \equiv x$, meaning that $w \in \lang(G, A) \cap \cc{x}{\equiv_R}$.

    \item[($\supseteq$)]
    Suppose that $A \deriv_G^n w$ for some $n \in \naturals$ and $w \equiv_R x$ for $x \in C$; it suffices to show that $A^x \deriv_{G_R}^* w$.
    The proof proceeds by induction on $n$.
    In the base, where $n = 1$, we have that $A \production w \in P$, and so $A^x \in V_R$ and $A^x \production w \in P_R$.
    Consequently, $A^x \deriv_{G_R}^* w$.

    For the inductive step, let $n > 1$, and assume the claim holds for $n' < n$.
    We then find that $A \deriv_G \alpha \deriv_G^{n-1} w$ for $A \production \alpha \in P$.
    In that case, we can write $\alpha = a_0A_0a_1A_1\cdots{}A_{k-1}a_k$ and $w = a_0w_0a_1w_1\cdots{}w_{k-1}a_k$ such that for $0 \leq i < k$ we have that $A_i \deriv_G^{n_i} w_i$ with $n_i \leq n-1$.
    For $0 \leq i < n$, let us write $x_i$ for the unique element of $C$ such that $x_i \equiv_R w_i$.
    By induction, we find for $0 \leq i < k$ that $A_i^{x_i} \deriv_{G_R}^* w_i$.
    Furthermore, note that $a_0x_0a_1x_1\cdots{}x_{k-1}a_k \equiv_R a_0w_0a_1w_1\cdots{}w_{k-1}a_k = w \equiv_R x$, and so we find that $A^x \production a_0A_0^{x_0}a_1A_1^{x_1}\cdots{}A_{k-1}^{x_{k-1}}a_k \in P_R$.
    In total, we have $A^x \deriv_{G_R}^* w$.
\end{itemize}

As for~\eqref{lemma:regular-intersection-construction:preservation}, it suffices to show that if $A^x \in V_R$ and $y, z \in \lang(G_R, A^x)$, then $y \equiv_{\lang(G_R)} z$.
To this end, suppose that $u, v \in \alphabet^*$ such that $uyv \in \lang(G_R)$.
In that case, $y, z \in \lang(G, A) \cap \cc{x}{\equiv_R}$, and thus $y \equiv_{\lang(G)} z$ as well as $y \equiv_R z$.
Consequently, $uzv \in \lang(G, A)$ and $uzv \in R$, and thus $uzv \in \lang(G_R)$.
Symmetrically, $uzv \in \lang(G_R)$ implies that $uyv \in \lang(G_R)$; we can thus conclude that $y \equiv_{\lang(G_R)} z$.
\end{proof}

\end{document}